\documentclass[letterpaper, 10 pt, conference]{ieeeconf}  

\IEEEoverridecommandlockouts                              

\usepackage{tabularx} 
\usepackage{booktabs} 
\usepackage{ragged2e} 
\usepackage{amsmath}     
\usepackage{amssymb}     
\usepackage{mathtools}   
\let\labelindent\relax
\usepackage{enumitem} 
\usepackage{xcolor}  
\usepackage{colortbl}  
\usepackage{graphicx}
\usepackage{subcaption} 

\usepackage[table]{xcolor}
\definecolor{oursred}{HTML}{ff6361}
\newcommand{\ourscell}[1]{{\scriptsize\textcolor{oursred}{#1}}}

\definecolor{deepred}{HTML}{FF6361}  

\usepackage{hyperref}
\usepackage{colortbl} 
\usepackage{multirow}
\makeatletter
\newcommand{\thickhline}{%
    \noalign {\ifnum 0=`}\fi \hrule height 1pt
    \futurelet \reserved@a \@xhline
}
\newcommand{\ci}[1]{\tiny{\textcolor{gray}{~($\pm #1$)}}}
\makeatother
\usepackage[export]{adjustbox}
\usepackage{bm}
\usepackage{xcolor}
\usepackage{listings}
\definecolor{mygray}{gray}{.9}
\usepackage{subcaption} 
\usepackage{makecell}
\usepackage{graphicx}
\usepackage{caption} 
\usepackage{cuted}
\usepackage{placeins}
\let\labelindent\undefined
\newlength\labelindent

\renewcommand{\arraystretch}{1.25} 

\usepackage{amsthm}

\title{\LARGE \bf CMR: Contractive Mapping Embeddings for Robust Humanoid Locomotion on Unstructured Terrains}

\author{%
Qixin Zeng\textsuperscript{1},
Hongyin Zhang\textsuperscript{2},
Shangke Lyu\textsuperscript{3},
Junxi Jin\textsuperscript{1} \\
Donglin Wang\textsuperscript{2,*},
Chao Huang\textsuperscript{1,*}
\thanks{%
\textsuperscript{1}University of Southampton;
\textsuperscript{2}Westlake University;
\textsuperscript{3}Nanjing University.
\textsuperscript{*}Corresponding author.
}%
}

\newtheorem{theorem}{Theorem}

\usepackage{romannum}

\begin{document}

\maketitle
\thispagestyle{empty}
\pagestyle{empty}

\begin{abstract}
Robust disturbance rejection remains a longstanding challenge in humanoid locomotion, particularly on unstructured terrains where sensing is unreliable and model mismatch is pronounced. While perception information, such as height map, enhances terrain awareness, sensor noise and sim-to-real gaps can destabilize policies in practice. In this work, we provide theoretical analysis that bounds the return gap under observation noise, when the induced latent dynamics are contractive. Furthermore, we present \textbf{C}ontractive \textbf{M}apping for \textbf{R}obustness (CMR) framework that maps high-dimensional, disturbance-prone observations into a latent space, where local perturbations are attenuated over time. Specifically, this approach couples contrastive representation learning with Lipschitz regularization to preserve task-relevant geometry while explicitly controlling sensitivity. Notably, the formulation can be incorporated into modern deep reinforcement learning pipelines as an auxiliary loss term with minimal additional technical effort required. Further, our extensive humanoid experiments show that CMR potently outperforms other locomotion algorithms under increased noise. 
\end{abstract}

\section{INTRODUCTION}
Humanoid robots hold immense potential for operating in human environments~\cite{gu2025humanoid}, especially with the development of whole-body control, which has expanded feasible task sets ~\cite{he2024learning,he2024omnih2o}. Alternatively, achieving stable locomotion remains challenging due to terrain irregularities, actuation perturbations and perception noise~\cite{bin2025survey}. In addition, deployment in unstructured environments typically requires exteroceptive sensing using radar, cameras, or height map~\cite{long2024learning,wang2025beamdojo,xie2025humanoid}. In these extreme scenes, sensing degradation and model mismatch make robustness more critical~\cite{miki2024learning}.

Approaches to robust humanoid locomotion broadly fall into two categories: model-based control and learning-based methods. On the one hand model-based approaches, such as the ZMP~\cite{nakaura2002balance}, LIP~\cite{1041633}, SLIP~\cite{8423425}, and SRBD controllers~\cite{Li2023Dynamic} offer interpretability but rely on restrictive assumptions, such as fixed CoM height and simplified centroidal dynamics, limiting their effectiveness on variable terrains~\cite{Pontón2020Efficient,Winkler2018Gait}. Moreover, model-based controllers solve deterministic optimal control and thus struggle with measurement uncertainty, which may lead to degraded performance under noisy perceptions and imprecise motor actuation~\cite{he2025attention}.
In contrast, learning-based methods have emerged as powerful tools by leveraging large-scale parallel simulations~\cite{rudin2022learning}. Notably, modern policy optimization~\cite{radosavovic2024real} can synthesize more complex behaviors, particularly in the context of Reinforcement Learning (RL) in humanoid robots. Previous work~\cite{zhuang2024humanoid,Radosavovic2023Real-world} introduced proprioception-only internal embeddings via contrastive learning, enabling robust locomotion without exteroception. Since solely relying on proprioceptive observations is insufficient to achieve a robust sim-to-real transfer, Wang et al.~\cite{wang2024toward} evaluated different types of explicit and implicit estimations incorporated in humanoid robot control. Even when augmenting environment estimation, maintaining robustness on more precise and challenging terrains such as stairs, stepping stones, and balance beams remains nontrivial. For these  environments previous work has utilized cameras~\cite{zhuang2024humanoid} and radar~\cite{long2024learning,wang2025beamdojo,he2025attention}, yet these exteroception devices introduce additional noise~\cite{bin2025survey}. 

\begin{figure}[t]
\centering
\includegraphics[width=\linewidth]{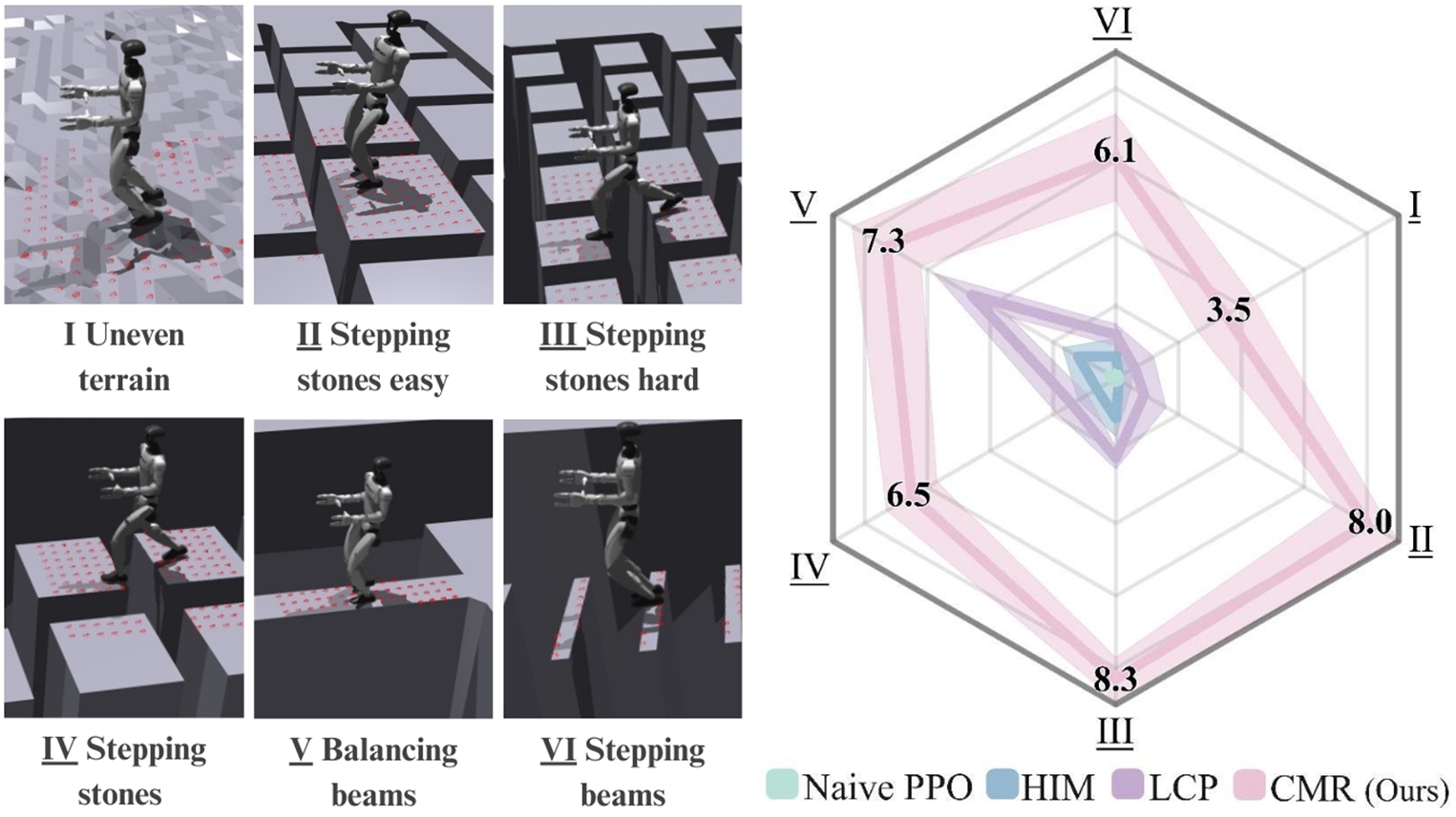}
\caption{The left panel illustrates diverse types of challenging terrains I–VI for humanoid locomotion tasks under \textcolor{deepred}{\textbf{noisy}} observations. While the radar plot on the right shows that {CMR{ \ourscell{(\textbf{Ours})}}} consistently outperforms the baselines under Noise Case \Romannum{3}. More results are further reported in Fig.~\ref{fig:noise_level_comparison}.
}   
\label{fig:terrains sketch}
\end{figure}


Additionally, existing approaches that ensure robustness against observation noise have profound limitations. Traditional methods mitigate noise effects through low-pass filters~\cite{ji2022hierarchical,li2021reinforcement} or smoothness rewards. Since such techniques require tuning overhead to balance smooth behavior with task completion, and filters reduce exploration capacity, they impede controller's transferability to diverse robotic systems.
A recent study, LCP~\cite{chen2024learning} applies Lipschitz constraints via differentiable gradient penalties but remains limited to basic walking behaviors on low-complexity terrains. Crucially, systematically addressing noise robustness in complex humanoid locomotion remains an open research problem.

To this end, we propose a novel \emph{Contractive Mapping for Robustness} (CMR) approach, which maps states into a latent space where disturbances are systematically attenuated rather than amplified along trajectories. Our framework combines contrastive representation learning for semantic preservation with Lipschitz-constrained control for bounded sensitivity, yielding disturbance-resilient latent dynamics. Furthermore, experiments partially shown in Fig.~\ref{fig:terrains sketch} demonstrate that CMR significantly exceeds other baseline algorithms under elevated observation noise in diverse terrains. Ablation studies further validate the effectiveness of its components. In sim-to-sim experiments, CMR enables zero-shot deployment and remains performance, providing evidence of generalization.

In summary, our contributions are three-fold:
\begin{itemize}
\item{  We  present CMR, a novel contractive-embedding framework that systematically attenuates observation perturbations while preserving task-relevant structure.}

\item{ We are the first to present contraction mapping theorem into learning-based humanoid locomotion task, establishing rigorous bounds on noise-induced return gaps.
}

\item{ We introduce a simple but effective training system that integrates seamlessly with normal deep RL pipelines. All Together, these advances enable robust humanoid locomotion across diverse and challenging terrains.
}
\end{itemize}
\section{Related Work}

\subsection{Reinforcement Learning for Humanoid Locomotion}

With large-scale parallel simulations~\cite{rudin2022learning} and modern policy optimization, deep RL can synthesize complex behaviors~\cite{radosavovic2024real}. 
Early work in humanoid locomotion primarily relied on proprioceptive feedback to avoid the complexity of visual processing. Previous work~\cite{radosavovic2024real} used observation and action history containing useful information about the world. Wang et al.\cite{wang2024toward} proved that learned state estimation is critical for enabling blind walking capability in real-world. While these approaches achieve impressive performance on relatively flat terrains. For more challenging settings including stairs and balancing beams, perception is so critical that humanoid robots can fully exploit the embodiment and perform tasks beyond locomotion on the plane~\cite{bin2025survey}. Recent studies have incorporated visual perception to handle complex terrains. Long et.~\cite{long2024learning}, who followed HIM~\cite{long2023hybrid} approach integrating internal models via contrastive learning to infer external states, used onboard elevation map to get ground-truth obstacle heights. BEAM DOJO~\cite{wang2025beamdojo} employs sampling-based foothold rewards with curriculum learning from flat to beams with robot-centric height map. Zhuang et al.~\cite{zhuang2024humanoid} presented an end-to-end approach using camera for humanoid parkour without motion priors. These methods enable complex terrain traversal but often suffer from perception noise. While CMR provides rigorous robust bounds for humanoid locomotion policies under noisy observations.

\subsection{Observation Noise in Legged Locomotion}
In real-world scenarios, legged robots often lack access to external states such as elevation mapping or contact forces. Moreover, the available onboard observations are potentially corrupted by noise. To address these challenges, prior work has leveraged mimic learning to compensate for the missing environmental information. Existing mimic-learning approaches can be broadly grouped into two frameworks: \textit{adaptation}~\cite{margolis2023walk,kumar2021rma} and \textit{teacher-student}~\cite{wu2023learning,margolis2024rapid,chen2020learning,lee2020learning}. Adaptation methods fine-tune a policy online to match the deployment environment, whereas teacher-student methods distill knowledge from a privileged policy as the teacher to a deployable policy as the student under noisy and partial observations~\cite{bogdanovic2022model}. Although these techniques mitigate the sim-to-real gap, a substantial performance drop typically remains between the teacher policy and the final deployable policy~\cite{Lee2024Learning}. Unlike the above mimic-learning frameworks, CMR purely learns a \emph{contractive state embedding} that provably limits noise propagation and integrates seamlessly with modern RL frameworks.

\subsection{Sensitivity Control and Robustness }
In control systems, ensuring solution existence and stability under constraints and perturbations is a fundamental robustness problem. Toward this goal, \textit{Jacobian regularization} and \textit{Lipschitz constraints} are widely used in system control~\cite{Daudin2021Optimal,Kim2020Hamilton-Jacobi}. On the one side, Jacobian-based methods directly penalize input-output sensitivity through gradient penalties \cite{varga2017gradient,achtibat2024attnlrp}, aiming to suppress rapid action fluctuations that could damage hardware. 
On the other side, Lipschitz constraints provide global continuity~\cite{mysore2021regularizing,9981812,pradhan2024near}, ensuring that small input perturbations lead to proportionally bounded output changes. This is because Lipschitz constraints can be verified using convex programming and semidefinite programming relaxations~\cite{fazlyab2019efficient}. Additionally, it is a critical property for maintaining stability in dynamic systems such as humanoid robots. For instance, practical implementation LCP \cite{chen2024learning} offers differentiable objectives compatible with deep learning frameworks. Different from these existing sensitivity control methods, CMR focuses on learning a contractive latent space where small disturbances are naturally attenuated over time.

\section*{Problem Formulation }

We model humanoid locomotion as a Partially Observable Markov Decision Process (POMDP) \(\langle \mathcal{S}, \mathcal{A}, f, r, \gamma,\mathcal{O} \rangle\), with the long horizon $H$. Here, \(\mathcal{S}\) encompasses the true states including joint angles, velocities, base pose, and exteroceptive features such as height map.
\(\mathcal{A}\) includes continuous motor commands across multiple Degrees of Freedom (DOF).
\(f: \mathcal{S} \times \mathcal{A} \to \Delta(\mathcal{S})\) models the nonlinear, stochastic dynamics arising from contact uncertainties and sensor noise. The Lipschitz constant of the system dynamics \(L_f\) reflects system stability~\cite{jin2020stability}.
\(r: \mathcal{S} \times \mathcal{A} \to \mathbb{R}\) encodes rewards for stability, progress, and robustness. The Lipschitz constant of the reward function \(L_r\) characterizes the robustness of the reward signal itself~\cite{de2024guaranteeing}.
\(\gamma \in [0,1)\) is the discount factor.
\(\mathcal{O}\) is the observation space, consisting of noisy exteroceptive and proprioceptive observations. Specifically, given the true state \(s_t \in \mathbb{R}^d\), the observed state \(\tilde{s}_t = s_t + \delta_s^t \in \mathcal{O}\), where \(\delta_s^t \sim \mathcal{U}([-\zeta, \zeta]^d)\) is zero-mean uniform noise~\cite{gu2025robust}. 
Our goal is to derive a robust policy in noisy conditions:
\[
\pi^* = \arg\max_{\pi} \ \mathbb{E}_{\pi} \left[ \sum_{t=0}^{\infty} \gamma^t r(\tilde{s}_t, a_t) \right],
\]
where transitions are \(\tilde{s}_{t+1} \sim f(\tilde{s}_t, a_t)\). Define the policy-gradient feature as  $g(s,a) := \nabla\log \pi\!\left(a \mid s\right)$.
Policy value is \(J(\pi) = \mathbb{E}_{\pi} \left[ \sum_{t=0}^{\infty} \gamma^t r(\tilde{s}_t, a_t) \right]\). Robustness is assessed via the return gap \(J(\pi^*) - J(\pi)\)~\cite{Kiemel2021Learning}, measuring performance degradation from perturbations.


\section{Methodology}
\begin{figure}[t]
\centering
\includegraphics[width=\linewidth]{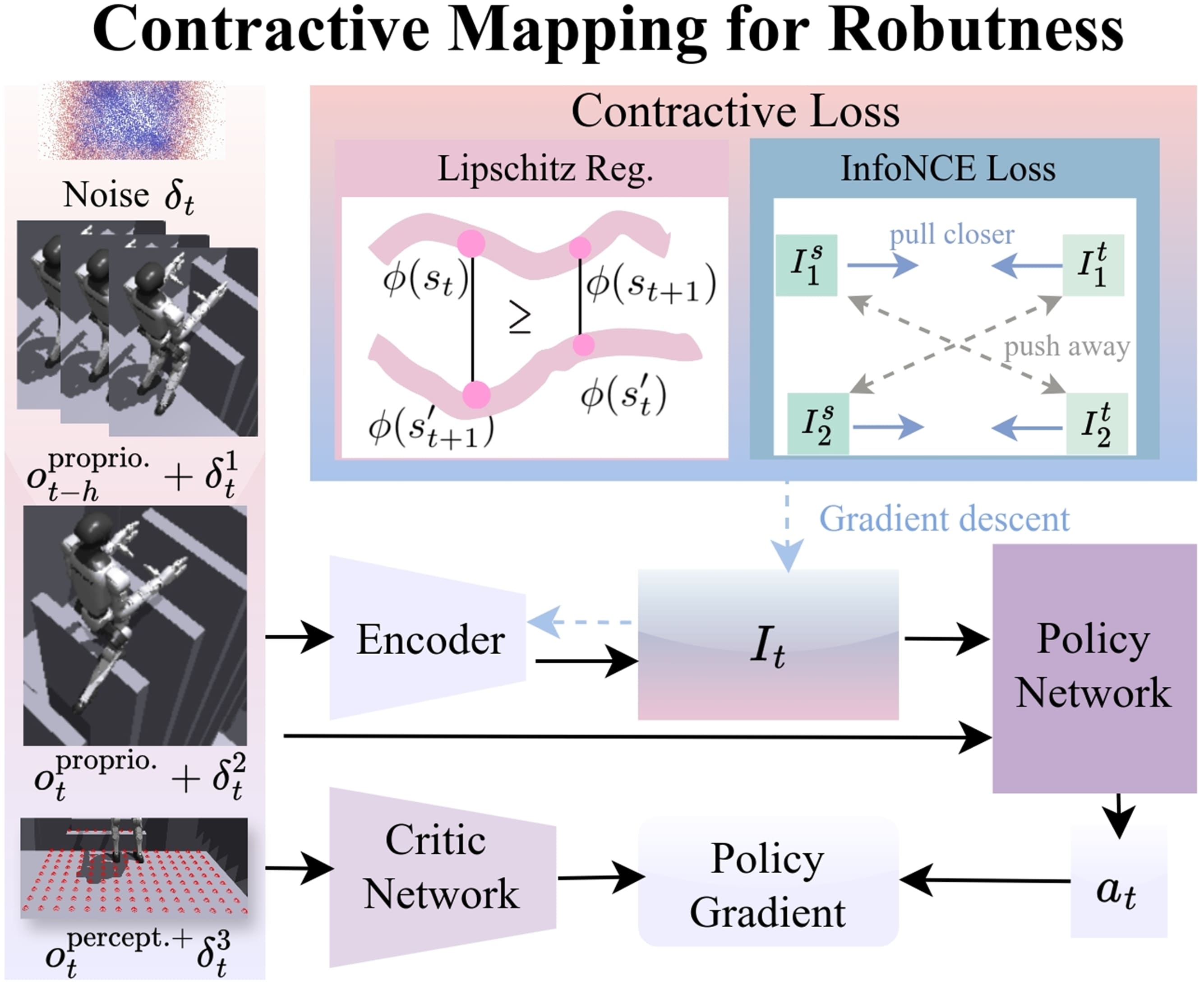}
\caption{\textbf{Overview of the CMR framework. }Noisy observations, including current and historical proprioception ($o_t^{\text{proprio}}+\delta_t^1$, $o_{t-h}^{\text{proprio}}+\delta_t^2$) and perception ($o_t^{\text{percept}}+\delta_t^3$), are encoded into the contractive embedding $I_t$. Training $I_t$ uses a Lipschitz regularization term and a contrastive objective. As a result, the trained contractive embedding $I_t$ is compatible with deep-RL pipelines as an additional input.}
\label{fig:Framework of CMR}
\end{figure}
This section details the proposed CMR framework. We begin by analyzing the performance degradation caused by observation noise. We then present the theoretical foundation for the contractive embedding space, followed by the formulation of optimization objectives. Finally, the specifics of training and implementation are presented. All detailed derivations are provided in the Appendix.

\subsection{Motivating Analysis}
\begin{theorem}[Return Gap under Observation Noise]
Assume perturbed observations $\tilde{s}_t = s_t + \delta_s^t$ with $\|\delta_s^t\| \leq \delta_{max} $, current policy \( \pi \)  with a bounded policy Jacobian $\|\nabla_s \pi_(s)\| \leq M$ \cite{Gannot2019A}, and optimized policy \( \pi^* \). Then:
\label{Return Gap under Observation Noise}
\[
J(\pi^*) - J(\pi) \leq \mathcal{O}(H L_r {L_f}^H) \cdot M  \delta_{max}.
\]
\end{theorem}
Humanoid robots exhibit high DOF, strongly coupled, multimodal dynamics, making conventional low-dimensional representation learning methods inadequate for robustness and transient behavior modeling~\cite{Radosavovic2023Real-world, wang2024toward}. In contrast, the embedding in \textbf{Theorem 1} bounds policy sensitivity by $M$ and ensures observation noise $\delta_s^t$ degrades $J(\pi)$ only proportionally, yielding a principled representation. Otherwise, for systems with $L_f \ge 1$~\cite{1241826}, the $L_f^H$ term causes exponential error growth over horizon $H$, so minor noise leads to severe degradation. This motivates a more geometric, contractive embedding space where noise is inherently attenuated without $L_f^H$ amplification, thereby improving robustness.

\subsection{Analysis of Contractive Embedding Space}
\begin{theorem}[Return Gap in Contractive Embedding Space]
\label{Return Gap in Contractive Embedding Space}
To enhance robustness in humanoid locomotion, we map states into a contractive latent space via an embedding projection $\phi$, where disturbances attenuate over time. Let $\phi$ be a contractive embedding with $0 < \kappa < 1$,  \( a_t \sim \pi(\cdot|s_t) \), \( a_t^* \sim \pi^*(\cdot|s_t) \), $\epsilon_t=\|\phi(f(s_t, a_t)) - \phi(f(s_t, a_t^*))\|$, current policy \( \pi \), and optimized policy \( \pi^* \) satisfy \( \mathbb{E}[\|a_t - a_t^*\|] \leq \eta \) for all \( t \)~\cite{8344843} satisfying:
\[
\|\phi(s_{t+1}) - \phi(s_{t+1}')\| \leq \kappa \|\phi(s_t) - \phi(s_t')\| + \epsilon_t,
\]
then,
\[
J(\pi^*) - J(\pi) \leq \mathcal{O}\left( \frac{\eta}{1 - \kappa} \right).
\]
\end{theorem}

\textbf{Global stability through contraction.} In this contractive latent space, state transitions become more regular and predictable, as the contraction property ensures that differences between trajectories shrink over time at rate $\kappa$. This regularity directly simplifies policy optimization by reducing the effective complexity of the underlying dynamics, enabling deep RL algorithms to learn stable and efficient policies. Moreover, \textbf{Theorem 2} quantifies this benefit: the return gap is bounded by $\mathcal{O}\left(\frac{\eta}{1-\kappa}\right)$, where $\eta$ bounds the expected action discrepancy and $\frac{1}{1-\kappa}$ captures the accumulation of errors under contraction. This analysis reveals how $\kappa$ precisely governs the trade-off between expressiveness (larger $\kappa$) and robustness (smaller $\kappa$), providing a principled way to limit performance degradation from observation noise.

\textbf{Long-horizon robustness against error growth.} Contraction embeddings constrain similar states to remain close in the latent space, thereby suppressing the propagation of errors across time steps $H$. It is worthwhile to note that, compared with \textbf{Theorem 1}, this expression is independent of horizon length $H$, making it significantly more tractable for long-horizon control tasks.

\textbf{Align high‑dimensional observations with policy optimization.} The high-dimensional, nonlinear observation spaces (e.g., images, depth maps) in real-world humanoid locomotion, coupled with the robots' strongly coupled dynamics and vast state spaces, fundamentally lack geometric regularities. Here contractive representation learning is crucial: it maps complex observations into a low-dimensional latent space with $\kappa<1$ dynamics, thereby preserving long-horizon robustness and enabling policy optimization with simplified and geometry-aware dynamics.

Accordingly, let the noise be $\delta_s^t \sim \mathcal{U}([-\zeta, \zeta]^d)$. Then the bias and variance of the policy‑gradient feature $g(s)$ satisfy
\begin{equation}
\label{viars}
\bigl\|\mathbb{E}_\delta\big[g(\tilde{s})\big]-g(s)\bigr\| \;\le\; \kappa L_\pi\, \zeta \sqrt{\tfrac{d}{3}},
\end{equation}
\begin{equation}
\operatorname{Var}_\delta\!\big[g(\tilde{s})\big] \;\le\; \frac{\zeta^2}{3}\,\kappa^2 L_\pi^2\, d.
\end{equation}
These bounds establish that, under uniform noise, both the bias and variance of $g(s)$ are explicitly bounded, thereby enhancing policy robustness. Nevertheless, we acknowledge that \textbf{Theorem 2} may yield suboptimal performance compared to \textbf{Theorem 1} when the contraction strength $\eta$ is too excessive, as constrained by $\mathbb{E}[|a_t - a_t^*|] \le \eta$, which may lead to task and environment relevant information loss.

\subsection{Objective Functions for Contractive Embeddings}
Grounded in the foregoing theory, our objectives jointly promote \textbf{dynamic contraction} while preserving sufficient \textbf{semantic discrimination}, yielding an ideal embedding space satisfying \textbf{Theorem 2}. For dynamic contraction, we introduce a Lipschitz regularization term explicitly constraining the evolution of differences in the embedding space,
\begin{equation}
\mathcal{L}_{\text{Lipschitz}} = \mathbb{E}_{(s_t, s'_t, s_{t+1}, s'_{t+1})} \left[ {S}_{t+1} - \kappa^2 {S}_t \right]_+,
\end{equation}
where ${S}_t = \|\phi(s_t) - \phi(s'_t)\|^2$, \( [x]_+ = \max(x, 0) \) ensures the penalty applies only when the contraction condition is violated. $\mathcal{L}_{\text{Lipschitz}}$ aims to build the contractive mapping space $\phi(s)$ with $0 < \kappa < 1$. It also endows the policy with inherent self-stabilizing properties, enabling the robot to intrinsically maintain performance under external disturbances. Complementarily, semantic discrimination is attained by contrastive learning through the InfoNCE loss~\cite{oord2018representation}:
\begin{equation} 
\label{InfoNCE}
\mathcal{L}_{\text{InfoNCE}} = - \sum_{i=1}^N \log \frac{\text{w}_{i,j^+}}{\text{w}_{i,j^-} + \sum_{j=1}^N \text{w}_{i,j^+}},
\end{equation}
where $\text{w}_{i, j^{\sigma}} = \exp\!\left({\mathrm{sim}\big(\phi(s_i), \phi(s_{j^{\sigma}})\big)} /{\tau}\right)$, $ \sigma \in \{+, -\}$, $s_i^+$ and \( s_j^- \) respectively represent a positive and negative sample from the trajectory as \( s_i \) respectively, and \( \tau \) is the temperature parameter. This objective avoids excessive contraction that loses information in the contractive embedding space $\phi(s)$. We optimize the humanoid locomotion task's policy $\pi$ with Proximal Policy Optimization (PPO)~\cite{schulman2017proximal}, leveraging its parallel training capability. In addition to PPO optimization objective $\mathcal{L}_{\text{PPO}}$, we set $\lambda$ as a trade off between semantic discrimination and dynamic contraction.  Thus the total training objective $\mathcal{L}_{\text{CMR}}$:
\begin{equation}
\mathcal{L}_{\text{CMR}} = \mathcal{L}_{\text{InfoNCE}} + \lambda \cdot \mathcal{L}_{\text{Lipschitz}} + \mathcal{L}_{\text{PPO}}.
\label{equ:loss}
\end{equation}  
This formulation ensures that the embedding space $\phi(s)$ simultaneously preserves task-relevant information while guaranteeing the contraction properties necessary, making the CMR framework deployable in humanoid locomotion task.

\subsection{Training Settings and Sim-to-sim Transfer}
The full observation space at each time step $t$ is defined as $ \tilde{s}_t = \left\{ c_t,\ o^{\text{proprio}}_t,\ o^{\text{percept}}_t,\ a_{t-1}\right\}$. The noisy observation space $\tilde{s}_t$ consists of four components: command inputs $c_t \in \mathbb{R}^3$ specifying desired velocity commands (longitudinal velocity $v_x$, lateral velocity $v_y$, and angular yaw velocity $\omega_{\text{yaw}}$); proprioceptive observations $o^{\text{proprio}}_t \in \mathbb{R}^{64}$ comprising base angular velocity $\omega_t \in \mathbb{R}^3$, joint positions $\theta_t \in \mathbb{R}^{29}$, and joint velocities $\dot{\theta}_t \in \mathbb{R}^{29}$; perceptive observations $o^{\text{percept}}_t \in \mathbb{R}^{15 \times 15}$ representing an egocentric elevation map centered around the robot, covering a $1.5\,\text{m} \times 1.5\,\text{m}$ area with $0.1\,\text{m}$ resolution, and the previous time step's action $a_{t-1} \in \mathbb{R}^{12}$ or $a_{t-1} \in \mathbb{R}^{29}$  providing temporal continuity. In the lower-body control setting, only the 12 actuated lower-body joints receive time-varying references from the policy; the remaining upper-body joints remain at their default setpoints. For completeness, we also consider a whole-body variant in which the action space is expanded to $a_t \in \mathbb{R}^{29}$ and consequently $a_{t-1}\in\mathbb{R}^{29}$ in the observation. The noise magnitude is governed by a global noise level $\alpha$ that uniformly scales the predefined base noise across observation channels. To mitigate the sim-to-real gap, we inject noise exclusively into the perception observations during training following Beam Dojo\cite{wang2025beamdojo}, with a moderate $\alpha = 1$. To stress-test robustness, all noise configurations are further amplified in our experiments.

The total training loss combined has detailed in Equ~\ref{equ:loss}. We trained the policy using a double critic framework following the heuristic of Beam Dojo~\cite{wang2025beamdojo}, on a single NVIDIA A100 GPU. Initially, ground-truth environmental data in \textit{Isaac Gym }simulator is employed to train for computational efficiency. To further validate its generalization, the policy trained is transferred into \textit{MuJoCo} under noisy disturbtions.
Additionally, key hyperparameters used in the training process are listed in the appendix.

\section{experiments}
This section presents complete evaluations. We aim to answer the following questions Is CMR more effective and robust than the baselines? (2) Does each component of CMR contribute to improving its performance? (3) How does CMR generalize in zero-shot transfer scenarios?

\begin{table*}[!ht]

    \caption{Different algorithms' performance comparison under Noise Cases I/II.
    }
    \label{noisy performance}
    \begin{center}
        \renewcommand\arraystretch{1.2}

        \resizebox{\linewidth}{!}{
        \begin{tabular}{c|c|c|c|c|c|c}
        \hline \thickhline
        \multirow{2}{*}{\textbf{Experiment Settings}} & 
        \multicolumn{6}{c}{\textbf{Environment: Uneven Terrain \& Disturbance Configurations: Noise Case I}}  \\
        \cline{2-7}
        & \textbf{Joint Power (\(\downarrow\))} & \textbf{Action Rate (\(\downarrow\))} & \textbf{Action Smoothness (\(\downarrow\))} & \textbf{Dof Vel (\(\downarrow\))} & \textbf{Dof Acc (\(\downarrow\))} & \textbf{Distance (\(\uparrow\))}  \\
        \hline \hline
HIM & 
$1971.9$\ci{712.2} & $10.6$\ci{4.7} & $10.8$\ci{4.8} & $170.4$\ci{66.9} & $430428.8$\ci{170306.7} & $2.3$\ci{0.8} \\

LCP & 
$1333.3$\ci{1526.7} & $4.1$\ci{2.4} & $164.8$\ci{861.3} & $113.1$\ci{126.4} & $282492.8$\ci{242580.8} & $5.8$\ci{1.8} \\

Naive PPO & 
$13618.3$\ci{1105.7} & $27.0$\ci{2.2} & $5431.7$\ci{617.4} & $1203.9$\ci{95.5} & $2128629.0$\ci{151030.2} & $0.1$\ci{0.0} \\

{\textbf{CMR{ \ourscell{(Ours)}}}} & 
$\mathbf{1120.7}$\ci{277.5} & $\mathbf{4.1}$\ci{1.4} & $\mathbf{4.5}$\ci{1.6} & $\mathbf{84.5}$\ci{24.6} & $\mathbf{213465.3}$\ci{69684.6} & $\mathbf{7.0}$\ci{1.3} \\
        \hline 
        \end{tabular}
        }

        \vspace{0.2em}  

        \resizebox{\linewidth}{!}{
        \begin{tabular}{c|c|c|c|c|c|c}
        \hline 
        \multirow{2}{*}{\textbf{Experiment Settings}} & 
        \multicolumn{6}{c}{\textbf{Environment: Uneven Terrain \& Disturbance Configuration: Noise Case II}} \\
        \cline{2-7}
        & \textbf{Joint Power (\(\downarrow\))} & \textbf{Action Rate (\(\downarrow\))} & \textbf{Action Smoothness (\(\downarrow\))} & \textbf{Dof Vel (\(\downarrow\))} & \textbf{Dof Acc (\(\downarrow\))} & \textbf{Distance (\(\uparrow\))}  \\
        \hline \hline
HIM & 
$1486.2$\ci{885.4} & $8.1$\ci{5.6} & $8.6$\ci{5.9} & $133.9$\ci{84.2} & $372216.0$\ci{235641.8} & $1.3$\ci{0.6} \\

LCP & 
$\mathbf{976.3}$\ci{795.0} & $3.3$\ci{2.6} & $13.3$\ci{248.5} & $83.2$\ci{75.0} & $217245.5$\ci{180823.6} & $4.9$\ci{1.7} \\

Naive PPO & 
$13943.9$\ci{383.7} & $23.2$\ci{1.3} & $4103.4$\ci{234.1} & $1135.8$\ci{38.3} & $1874354.3$\ci{80916.9} & $0.2$\ci{0.0} \\

{\textbf{CMR{ \ourscell{(Ours)}}}} & 
$1012.3$\ci{350.3} & $\mathbf{3.2}$\ci{1.5} & $\mathbf{3.4}$\ci{1.6} & $\mathbf{77.8}$\ci{31.3} & $\mathbf{184686.9}$\ci{82814.0} & $\mathbf{7.2}$\ci{1.4} \\

        \hline \thickhline
        \end{tabular}
        }

    \end{center}
    \label{table:uneven}
    \vspace*{-1pt}
\end{table*}

\begin{table*}[!h]
    \caption{Different algorithms' performance comparison under dual command configurations.
    }
    \label{table:command ood}
    \begin{center}
        \renewcommand\arraystretch{1.2}

        \resizebox{\linewidth}{!}{
        \begin{tabular}{c|c|c|c|c|c|c}
        \hline \thickhline
        \multirow{2}{*}{\textbf{Experiment Settings}} & 
        \multicolumn{6}{c}{\textbf{Environment: Uneven Terrain (Walking Random) \& Disturbance Configurations: Noise Case I}} \\
        \cline{2-7}
        & \textbf{Joint Power (\(\downarrow\))} & \textbf{Action Rate (\(\downarrow\))} & \textbf{Action Smoothness (\(\downarrow\))} & \textbf{Dof Vel (\(\downarrow\))} & \textbf{Dof Acc (\(\downarrow\))} & \textbf{Distance (\(\uparrow\))}  \\
        \hline \hline

HIM & 
$3523.4$\ci{950.9} & $9.9$\ci{3.3} & $10.1$\ci{3.3} & $162.4$\ci{43.4} & $414386.5$\ci{110425.4} & $2.1$\ci{0.8} \\

LCP & 
$2451.5$\ci{3230.6} & $\mathbf{3.8}$\ci{1.3} & $213.6$\ci{969.4} & $116.6$\ci{134.1} & $282361.0$\ci{230976.4} & $5.6$\ci{1.7} \\

Naive PPO & 
$26832.2$\ci{1408.3} & $27.5$\ci{1.7} & $5608.1$\ci{354.8} & $1230.6$\ci{57.9} & $2167369.3$\ci{104718.8} & $0.08$\ci{0.02} \\

{\textbf{CMR{ \ourscell{(Ours)}}}} & 
$\mathbf{1854.7}$\ci{345.6} & $3.9$\ci{0.8} & $\mathbf{4.4}$\ci{1.0} & $\mathbf{81.0}$\ci{13.8} & $\mathbf{206715.1}$\ci{39190.9} & $\mathbf{6.1}$\ci{1.5} \\
        \hline 
        \end{tabular}
        }

        \vspace{0.3em}  

        \resizebox{\linewidth}{!}{
        \begin{tabular}{c|c|c|c|c|c|c}
        \hline 
        \multirow{2}{*}{\textbf{Experiment Settings}} & 
        \multicolumn{6}{c}{\textbf{Environment: Uneven Terrain (Walking Random OOD) \& Disturbance Configurations: Noise Case II}} \\
        \cline{2-7}
        & \textbf{Joint Power (\(\downarrow\))} & \textbf{Action Rate (\(\downarrow\))} & \textbf{Action Smoothness (\(\downarrow\))} & \textbf{Dof Vel (\(\downarrow\))} & \textbf{Dof Acc (\(\downarrow\))} & \textbf{Distance (\(\uparrow\))}  \\
        \hline \hline

HIM & 
$2505.0$\ci{558.9} & $12.9$\ci{2.9} & $13.3$\ci{2.9} & $208.7$\ci{34.9} & $515972.8$\ci{86794.4} & $2.1$\ci{0.6} \\

LCP & 
$1771.8$\ci{1544.0} & $\mathbf{5.6}$\ci{1.1} & $182.0$\ci{868.4} & $157.4$\ci{112.8} & $390014.1$\ci{190081.3} & $5.0$\ci{1.3} \\

Naive PPO & 
$13936.9$\ci{1752.7} & $26.2$\ci{2.9} & $5212.2$\ci{864.7} & $1168.6$\ci{133.5} & $2074207.0$\ci{203178.2} & $0.09$\ci{0.05} \\

{\textbf{CMR{ \ourscell{(Ours)}}}} & 
$\mathbf{1395.6}$\ci{283.5} & $5.9$\ci{1.1} & $\mathbf{6.9}$\ci{1.5} & $\mathbf{123.5}$\ci{18.5} & $\mathbf{316543.9}$\ci{53703.2} & $\mathbf{6.2}$\ci{0.8} \\
        \hline \thickhline
        \end{tabular}
        }
    \end{center}
    \label{table:uneven}
    \vspace*{-1pt}
\end{table*}

\subsection{Evaluation Setups}
Baselines. We compare our approach against existing baselines in the context of humanoid locomotion. \textbf{HIM:} It uses a hybrid internal embedding representation~\cite{long2023hybrid}. Long et~\cite{long2024learning} integrate it into humanoid control. \textbf{LCP:} This method implements Lipschitz constraint with a gradient penalty~\cite{chen2024learning}. \textbf{Naive PPO:} It is a plain PPO baseline.

Disturbance Configurations. We consider two noise categories: proprioceptive and perceptual observation noise. Five disturbance scenarios are considered: \textbf{Noise Case \Romannum{1}:} High-level noise with $\alpha=2$ applied exclusively to the exteroceptive height map input.\textbf{ Noise Case \Romannum{2}:} Mild noise with $\alpha=0.01$ applied to proprioceptive and perceptive observations. \textbf{Noise Case \Romannum{3}-\Romannum{5}:} Noise applied to both observational sources with increasing intensity: moderate $\alpha=1$, high $\alpha=2$, and extreme $\alpha=3$ respectively.

Dual robot configurations are conducted on the Unitree G1: (1) a constrained 12-DOF setup with the upper body locked, and (2) a full-body configuration with all DOF active. This configuration enables comprehensive assessments of algorithms' robustness across different kinematic complexities.

Experiment metrics. We employ diverse metrics for evaluations, including Joint Power, Action Rate (control effort magnitude), Action Smoothness, DoF Velocity, DoF Acceleration, Distance Traveled, and Command Tracking Error.

\subsection{Main Results in Simulation}
 \textbf{Noise Rejection and Stability.} We evaluate the proposed method’s performance across multiple metrics under identical experimental settings (Table.~\ref{noisy performance}). Compared with the naive PPO baseline, incorporating the embedding components consistently improves overall performance. In the first scenario, under perceptual noise levels exceeding those in training configurations, controllers utilizing embedding components largely maintain stable and effective locomotion. In the second scenario, where mild noise affects all observations, most methods exhibit a noticeable drop in performance. Conversely, CMR consistently delivers superior performance and preserves locomotion quality across both scenarios. Although LCP demonstrates strength in joint power under Noise Case \Romannum{2}, it underperforms in other metrics. Notably, proprioceptive noise has a greater impact on performance degradation than perceptual noise in comparison between Noise Case \Romannum{1} and \Romannum{2}. Thus experiment below will focus more on the combined effects of both noise sources. Overall, our approach maintains the lowest power consumption, smoothest control actions, and the most stable internal dynamics, showcasing exceptional noise-rejection capability and locomotion stability. 

\textbf{Random Velocity and Out-of-distribution Conditions.} To evaluate our algorithms’ performance under random velocity commands, we conduct an additional experiment (Table.~\ref{table:command ood}). Even in the OOD setting, CMR achieves the longest travel distance and the lowest Action Smoothness score, while operating with the lowest joint power. CMR demonstrates optimal generalization to OOD scenarios coupled with excellent noise immunity, among all algorithms. We hypothesize that the contractive embedding space $\phi{(s)}$, inherently instills a self-stability attribute in the robot, optimized by the CMR objective. Consequently, the humanoid robot is expected to exhibit superior robustness. Intriguingly, when subjected to amplified noise, the robot executes locomotion tasks more thoroughly, yielding unexpectedly greater walking distances, as corroborated by data presented in Table.~\ref{noisy performance} and Table.~\ref{table:command ood}.

\textbf{Discriminative Embeddings via t-SNE Visualization.}
To elucidate the mechanisms underlying CMR's superiority, we employed t-SNE~\cite{maaten2008visualizing} to visualize latent embeddings from ours and HIM. In the presence of noise, CMR's trajectories (Fig.~\ref{fig:ts1}) exhibit significantly less deviation, validating its effective maintenance of robot self-stability and resilience in the contractive space. Across diverse terrains under amplified noise, CMR also yields more distinctly separated clusters (Fig.~\ref{fig:ts2}), underscoring its retained expressive capacity and environmental identification even within such a contractive space. Collectively, we find that CMR strikes an effective balance between dynamic contraction and semantic expressiveness and prove the advantages of \textbf{Theorem 2}, intrinsically explaining our approach's improved performance.

;
\begin{figure}[!htbp]
  \centering

  \begin{minipage}{.95\linewidth}
    \centering
    \includegraphics[width=\linewidth]{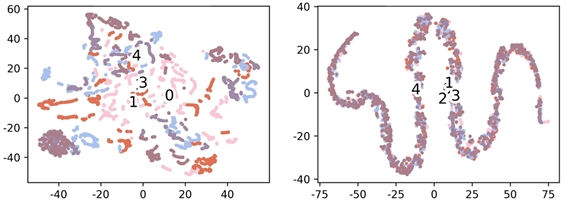}
    \captionof{figure}{Latent space visualizations of five trajectories from deployments of CMR \ourscell{(\textbf{left})} and HIM \ourscell{(\textbf{right})} under Noise Case I.}
    \label{fig:ts1}
  \end{minipage}

  \vspace{0.8em}

  \begin{minipage}{.95\linewidth}
    \centering
    \includegraphics[width=\linewidth]{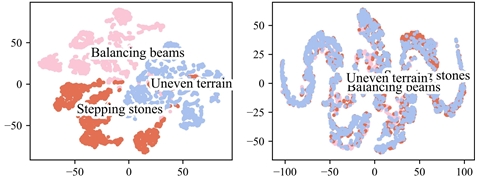}
    \captionof{figure}{Latent space visualizations of trajectories across terrains from deployments of CMR \ourscell{(\textbf{left})} and HIM \ourscell{(\textbf{right})}.}
    \label{fig:ts2}
  \end{minipage}

\end{figure}

\begin{figure}[t]
  \centering

  \begin{minipage}{\linewidth}
    \centering
    \includegraphics[width=\linewidth]{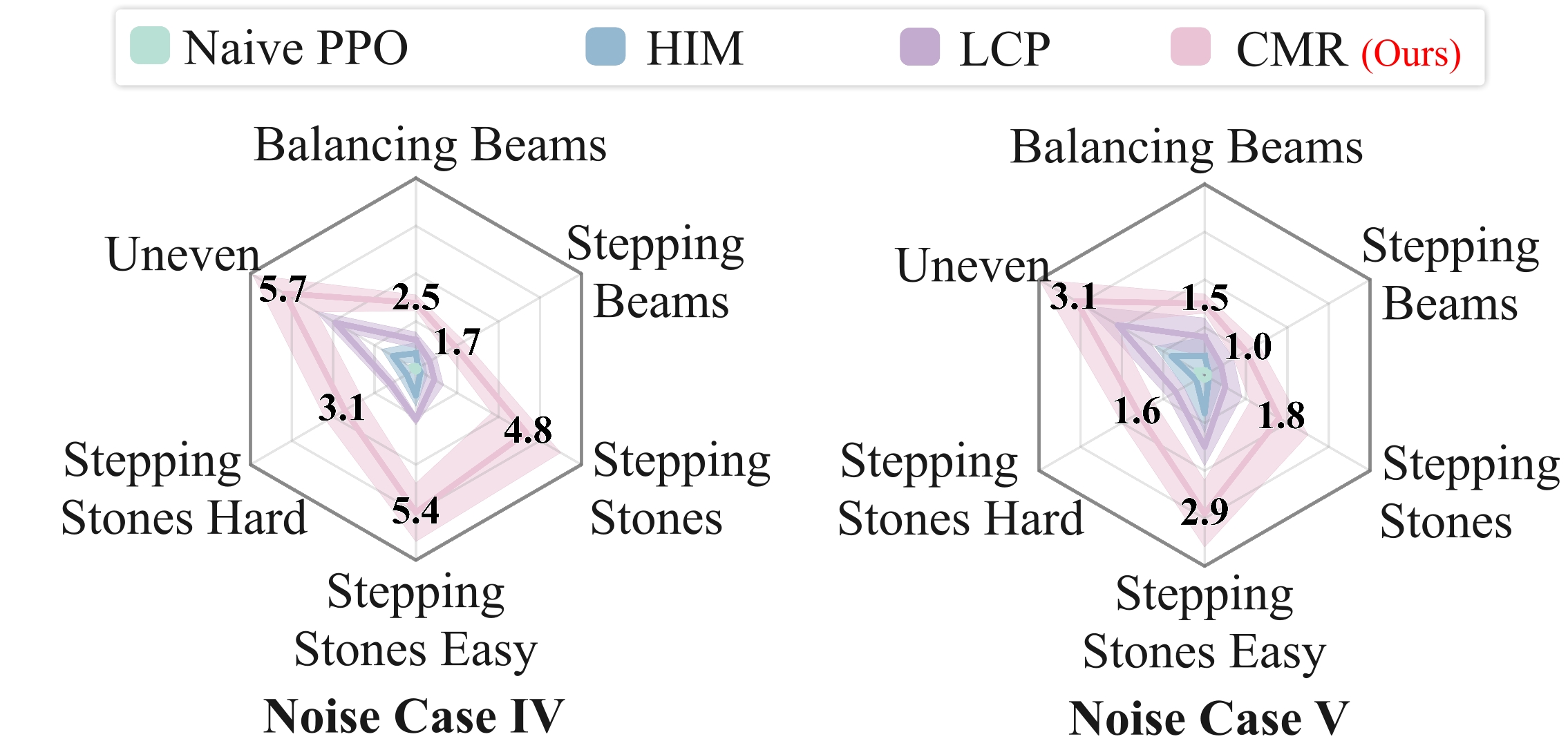}
    \captionof{figure}{Locomotion Distances for deployment of CMR \ourscell{(\textbf{Ours})} and baselines across varing terrains under Noise Case \Romannum{4}/\Romannum{5}.}
    \label{fig:noise_level_comparison}
  \end{minipage}

  \vspace{0.8em}

  \begin{minipage}{\linewidth}
    \centering
    \includegraphics[width=\linewidth]{figures/ablation_experiments_distance_results.png}
    \captionof{figure}{Ablation studies on key hyperparameters in CMR.}
    \label{fig:ablation_experiments_distance_results}
  \end{minipage}

\end{figure}


\textbf{Adaptability Across Diverse Terrain Types.} To further evaluate the environmental adaptation and the locomotion capability in unstructured terrains, we conduct experiments on six challenging and diverse terrains, such as uneven terrains, stepping stones, and balance beam (Fig.~\ref{fig:terrains sketch}). The results in Fig.~\ref{fig:terrains sketch} and Fig.~\ref{fig:noise_level_comparison} show that, across all terrain types and difficulty levels, the robot equipped with the proposed representation contractive embedding consistently covers markedly longer distances than all baseline approaches. 

\subsection{Ablation Studies}
An ablation study was conducted to evaluate the sensitivity of CMR to temperature $ \tau $, Lipschitz parameter $\lambda$ , and contractive factor $\kappa$, under the same simulation settings. As shown in Fig.~\ref{fig:ablation_experiments_distance_results}, the results collectively demonstrate a consistent trade-off between trajectory accuracy (Distance) and energy consumption (Joint Power) across all parameters. Increasing $\tau$ improves accuracy at the cost of higher energy use and reduced motion dynamics (DoF Vel, DoF Acc), while stronger Lipschitz regularization $\lambda$ enhances efficiency and smoothness but may limit responsiveness. Similarly, higher $\kappa$ values suppress oscillations and improve stability, albeit with increased energetic overhead. These inter dependencies validate our default configuration ($\tau$ = 0.04, $\lambda=0.05$, $\kappa=0.1$) as a balanced compromise for achieving robust and efficient robotic performance across diverse scenarios.

\subsection{Sim-to-sim Transfer}
To validate the generalization capability of CMR, we perform a sim-to-sim transfer by deploying the trained policies in the \textit{MuJoCo} simulation environment. Within \textit{MuJoCo}, we generate the command $c_t$ at each time step by sampling from a uniform distribution under Noise Case I. As shown in Table~\ref{table:sim2sim}, not only dose CMR track command better than HIM, but it also consumes pronouncedly less joint power, suggesting more energy-efficient control. These results demonstrate CMR’s strong zero-shot transfer ability.
\begin{table}[t]
    \begin{center}
        \renewcommand\arraystretch{1.1}
        \resizebox{\linewidth}{!}{
        \begin{tabular}{c|cc}
        \hline \thickhline
        \textbf{Metric} & \textbf{HIM} & \textbf{CMR \scriptsize{(Ours)}} \\
        \hline \hline
        Command Tracking Error (\(\downarrow\)) & 2.51\ci{0.22} & \textcolor{deepred}{\textbf{2.36}\ci{0.63}} \\
        DoF Vel (\(\downarrow\)) & 81.09\ci{18.68} & \textcolor{deepred}{\textbf{71.34}\ci{22.25}} \\
        Action Rate (\(\downarrow\)) & 2.84\ci{0.39} & \textcolor{deepred}{\textbf{2.67}\ci{0.35}} \\
        Joint Power (\(\downarrow\)) & 1465.43\ci{902.67} & 
        \textcolor{deepred}{\textbf{1023.32}\ci{576.21}} \\
        \hline \thickhline
        \end{tabular}
        }
    \end{center}
    \caption{Sim-to-sim transfer performance comparison between HIM and {CMR{ \ourscell{(\textbf{Ours})}}} across four evaluation metrics.}
    \label{table:sim2sim}
\end{table}

\section*{CONCLUSIONS}
In this paper, we propose a \textit{Contraction Mapping Robustness framework} (CMR) for robust humanoid locomotion in the presence of observation noise. CMR is the first framework to incorporate the contraction mapping theorem into learning-based robust humanoid locomotion. Its combination of contrastive representation learning and Lipschitz regularization enables explicit sensitivity control while remaining compatible with modern deep RL pipelines. Theoretical analysis further establish return-gap bounds under observation noise when latent dynamics are contractive. Comprehensive experiments indicate that CMR maintains a marked lead over alternative locomotion algorithms under amplified noise. Moreover, we are verifying transfer performance through real-robot experiments, thereby strengthening practical credibility. Moving forward, CMR will be augmented with more multi-modal perception such as depth cameras, to perform better in more precise and complex terrains.

\FloatBarrier

\section*{Appendix}

\textbf{(1) Proof of Theorem 1: Return Gap under Observation Noise}\\
\begin{proof}
Setup. We consider $\tilde{s}_t = s_t + \delta_s^t$ with $||\delta_s^t|| \leq \delta_{max}$, and $a_t = \pi_t(\tilde{s}_t)$. The only error source is observation noise. The policy Jacobian satisfies $||\nabla_s \pi_t(s)|| \leq M$.
Goal. Upper bound $J(\pi^*) - J(\pi)$.
The action deviation is bounded as follows: $\|a_t - a_t^*\| = \|\pi(\tilde{s}_t) - \pi(s_t^*)\| \leq M \|\tilde{s}_t - s_t^*\| 
                    \leq M (\|s_t - s_t^*\| + \|\delta_s^t\|) \leq M(x_t + \delta_{max}).$ By Lipschitz continuity of the dynamics $f$: $x_{t+1} = \|f(s_t, a_t) - f(s_t^*, a_t^*)\| 
            \leq L_f (\|s_t - s_t^*\| + \|a_t - a_t^*\|)
            \leq L_f (x_t + M \delta_{max}).$
Bounding $d_t$ via Unfolding. With $d_0=0$:
$$x_t \leq L_f M \delta_{max} \sum_{i=0}^{t-1} L_f^i 
        = \mathcal{O}(L_f^tM \delta_{max}).$$

Return Gap Bound. For discrete reward functions in humanoid locomotion, we approximate them using Gaussian kernel functions~\cite{wang2024toward} to ensure the Lipschitz continuity of the reward with a bounded constant $L_r$. At each step, using the $L_r$-Lipschitz continuity of the reward function $r$:
    $$|r(s_t, a_t) - r(s_t^*, a_t^*)| \leq L_r (x_t + \|a_t - a_t^*\|) 
                                   \leq L_r (x_t + M \delta_{max}).$$

Summing over horizon $H$ gives the final bound:
\[
J(\pi^*) - J(\pi) \leq \mathcal{O}\!\left(H L_r L_f^H\right) \cdot M \delta_{max} .
\]
\end{proof}
\textbf{(2) Proof of Theorem 2: Return Gap in Contractive Embedding Space}\\
\begin{proof}
The proof analyzes the deviation between a trajectory from policy $\pi$ and one from the optimal policy $\pi^*$. Let $s_t \sim P^\pi$ and $s_t' \sim P^{\pi^*}$ be the states at timestep $t$. We define the latent space deviation as $\eta_t := \phi(s_t) - \phi(s_t')$. From the theorem's premise, the latent dynamics follow $||\phi(s_{t+1}) - \phi(s_{t+1}')|| \leq \kappa ||\eta_t|| + \epsilon_t$, where $\epsilon_t = ||\phi(f(s_t, a_t)) - \phi(f(s_t, a_t^*))||$. We bound $\epsilon_t$ using the Lipschitz continuity of $\phi$ (constant $L_\phi$) and $f$ (constant $L_f$): $\epsilon_t \leq L_\phi ||f(s_t, a_t) - f(s_t, a_t^*)|| \leq L_\phi L_f ||a_t - a_t^*||.$\\

Substituting this back gives the single-step deviation bound:
$||\eta_{t+1}|| \le \kappa||\eta_{t}|| + L_{\phi}L_{f}||a_{t}-a_{t}^{*}||.$
Taking the expectation and unrolling with assumption $\mathbb{E}[||a_t - a_t^*||] \le \eta$: $\mathbb{E}[||\eta_t||] \le \frac{L_\phi L_f \eta}{1-\kappa} + \kappa^t ||\eta_0||.$ The per-step expected reward gap is bounded:
\begin{align*}
    &\mathbb{E}|r_t^\pi - r_t^{\pi^*}|\quad \le L_r \mathbb{E}[||\eta_t||] + L_r' \mathbb{E}[||a_t - a_t^*||] \\
    &\quad \le \frac{L_r L_\phi L_f}{1-\kappa}\eta+L_r' \eta\qquad+L_r \kappa^t ||\eta_0||.
\end{align*}
Summing over the horizon and assuming $\eta_0 = 0$, the total return gap is bounded independent of $H$:
\[
J(\pi^*) - J(\pi) \le \mathcal{O}\left(\frac{\eta}{1-\kappa}\right).
\] 
\end{proof}
\textbf{(3) Proof of Bias and Variance Bounds under Contraction}\\
The objective is to prove the following two inequalities for the policy feature $g$ under uniform observation noise $\delta \sim \mathcal{U}([-\zeta,\zeta]^d)$, given a contractive embedding $\phi$:
$$\|\mathbb{E}_\delta[g(\tilde{s})]-g(s)\| \leq \kappa L_\pi \zeta \sqrt{d/3},$$
$$ \mathrm{Var}_\delta[g(\tilde{s})] \leq \frac{\zeta^2 d}{3}\kappa^2 L_\pi^2. $$
$$\|g(\tilde{s}) - g(s)\| \le \kappa L_\pi \|\delta_s\|.$$

Proof of the Bias Bound.
We begin by bounding the norm of the expected difference.
\begin{align*}
    &\|\mathbb{E}_\delta[g(\tilde{s})] - g(s)\| 
    = \|\mathbb{E}_\delta[g(\tilde{s}) - g(s)]\| 
    \le \mathbb{E}_\delta[\|g(\tilde{s}) - g(s)\|] \\
    &\le \mathbb{E}_\delta[\kappa L_\pi \|\delta_s\|] 
    = \kappa L_\pi \mathbb{E}_\delta[\|\delta_s\|]
    \le \kappa L_\pi \zeta \sqrt{d/3}.
\end{align*}

Proof of the Variance Bound.
The variance of the policy feature $g(\tilde{s})$ is defined as $\mathrm{Var}_\delta[g(\tilde{s})]$.
\begin{align*}
    &\mathrm{Var}_\delta[g(\tilde{s})] 
    = \mathbb{E}_\delta[\|g(\tilde{s}) - \mathbb{E}_\delta[g(\tilde{s})]\|^2] 
    \le \mathbb{E}_\delta[\|g(\tilde{s}) - g(s)\|^2] \\
    &\le \mathbb{E}_\delta[(\kappa L_\pi \|\delta_s\|)^2] 
    = \kappa^2 L_\pi^2 \mathbb{E}_\delta[\|\delta_s\|^2]  = \kappa^2 L_\pi^2 \frac{\zeta^2 d}{3}.
\end{align*}

\textbf{(4) Key hyperparameters in the training process }
\begin{table}[h!]
\centering
\caption{Hyperparameters used for training.}
\label{tab:hyperparameters}
\renewcommand{\arraystretch}{0.65}
\setlength{\tabcolsep}{6pt} 
\begin{tabularx}{\linewidth}{c|c c}
\hline
\textbf{Category} & \textbf{Hyperparameter} & \textbf{Value} \\
\hline
General & Number of environments & 1200 \\
\hline
\multirow{5}{*}{Noise Scales} & DOF position & 0.05 \\
 & DOF velocity  & 1.0 \\
 & Linear velocity  & 0.1 \\
 & Angular velocity  & 0.5 \\
 & Gravity & 0.05 \\
\hline
\multirow{3}{*}{Height Measurements} & Extend probability & 0.6 \\
 & Vertical scale & 0.02 \\
 & Offset range & [-0.05, 0.02] \\
\hline
\multirow{3}{*}{Policy Network} & Actor hidden dimensions & {[}512, 256, 128{]} \\
 & Critic hidden dimensions & {[}512, 256, 128{]} \\
 & Activation function & elu \\
\hline
\multirow{2}{*}{Contractive Encoder} & Hidden dimensions & {[}256, 256, 128{]} \\
 & Activation function & elu \\
\hline
\multirow{3}{*}{Loss Coefficients} & Temperature ($ \tau $) & 0.04 \\
 & Kappa ($\kappa$) & 0.1 \\
 & Lambda ($\lambda$) & 0.05 \\
\hline
\end{tabularx}
\end{table}

\section*{ACKNOWLEDGMENT}

\bibliography{IEEEexample}  

\end{document}